\documentclass[letterpaper, 10 pt, journal,twoside]{IEEEtran}

\usepackage{stfloats}
\usepackage{graphicx}
\usepackage{amsmath,amsthm}
\usepackage{cite}
\usepackage{amssymb}
\usepackage{subfigure}
\usepackage{subfigmat}
\usepackage{epstopdf}
\usepackage{xcolor}
\usepackage{algorithm}
\usepackage{algorithmic}
\usepackage{comment}
\usepackage{multirow,multicol,threeparttable}
\usepackage{booktabs}
\usepackage{enumerate}
\usepackage{enumitem}

\usepackage[english]{babel}

\usepackage{tikz}
\usetikzlibrary{shapes,arrows,spy,positioning,patterns,shapes.geometric}
\usepackage{pgfplots}

\def\BibTeX{{\rm B\kern-.05em{\sc i\kern-.025em b}\kern-.08em T\kern-.1667em\lower.7ex\hbox{E}\kern-.125emX}}

\newtheorem{theorem}{Theorem}

\newtheorem{lemma}{Lemma}

\newcommand{\ignore}[1]{}  

\begin{document}

\title{An Approximation Algorithm for a Task Allocation, Sequencing and Scheduling Problem involving a Human-Robot Team}

\author{Sai Krishna Kanth Hari$^{1}$, Abhishek Nayak$^{2}$, and Sivakumar Rathinam$^{3}$%

\thanks{$^{1,2,3}$The authors are with the Department of Mechanical Engineering,
Texas A\&M University, College Station, TX, 77843.
Contact: {\tt\small srathinam@tamu.edu}}%
}

\maketitle

\begin{abstract}
This article presents an approximation algorithm for a task allocation, sequencing and scheduling problem involving a team of human operators and robots. Specifically, we present an algorithm with an approximation ratio as a function of the number of human operators ($m$) and the number of robots ($k$) in the team. The approximation ratios are $\frac{7}{2} -\frac{5}{2k}$, $\frac{5}{2} -\frac{1}{k}$ and $\frac{7}{2} -\frac{1}{k}$ when $m=1$, $m\geq k\geq 2$ and $k>m\geq 2$ respectively. We also present computational results to corroborate the performance of the proposed approximation algorithm. 
\end{abstract}

\section{Introduction}


This article considers a mission planning problem which arises when human operators located at a base station have to collaboratively work with mobile robots to visit a set of targets and perform inspection, classification or data collection tasks at the targets. Each target location is associated with a task. There is also a processing time associated with completing each task. Each human operator is only allowed to work on at most one task and each task requires at most one human operator to work on it at any time instant. The robots travel to the target locations, collaboratively work with the human operators to complete the tasks at the locations and return to their initial position. The allocation of targets (or tasks) and the sequence in which the targets must be visited for each robot is not known apriori. Also, when the number of human operators is less than the number of robots, it is possible that a robot has to wait at a target for a human operator to be available to collaboratively work on the task. The mission time of each robot will include its travel time, wait time and the processing times of the tasks at the targets assigned to it. The objective of the problem is to find a sequence of targets (tasks) for each robot to visit and schedule its tasks with the human operators such that each task is performed once by some robot, the scheduling constraints are satisfied and the maximum mission time of any of the robots is minimized. This problem is referred to as the Task Allocation, Sequencing and Scheduling Problem (TASSP). Refer to Fig. \ref{fig:samplepath} for an illustration of this problem.

\begin{figure}[h]
\centering{}
\includegraphics[scale=0.80]{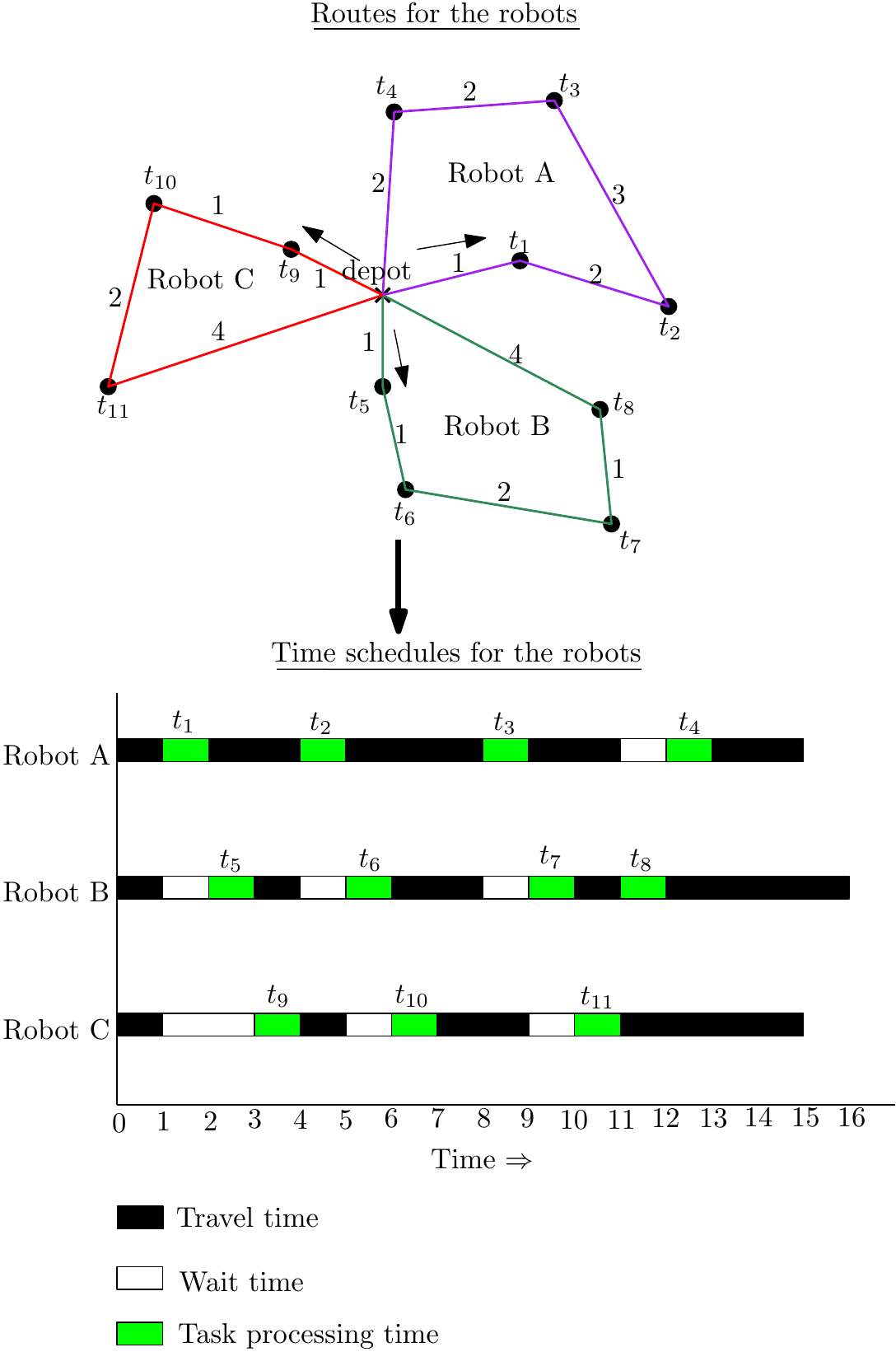}
\caption{An illustration of a feasible set of routes and time schedules for three robots. In this example, there is one human operator and the processing time of the task at any target is 1 unit. There are 11 targets ($t_1,\cdots,t_{11}$) and the time required to travel between any two adjacent targets in a route is also shown. The maximum mission time of any of the robots is 16.}
\label{fig:samplepath}
\vspace{-.5cm}
\end{figure}

TASSP naturally arises in data gathering or surveillance applications involving unmanned systems \cite{cummings2010} where few human operators are expected to collaboratively work with a large team of robots to perform specific tasks at the targets. Incorporating human factor requirements in unmanned vehicle routing has been of significant interest in military applications over the last decade\cite{Murray13,cummings12}. As discussed in the survey\cite{golden2018}, human operators in some unmanned applications should be allowed sufficient time to examine the real-time information at the target locations visited by the vehicles. In-site inspection and target classification applications where human input is critical require robots to interact with a human operator to classify the targets accurately in real-time\cite{Cai16}. Robots may also have manipulators that need to be tele-operated by a human operator. Specifically, in crop monitoring applications\cite{agrisurvey,agri1}, the cameras on a robot are tele-operated by a human operator to get the images with a desired resolution and field of view. 

In the absence of processing times for the tasks and scheduling constraints with the human operators, TASSP reduces to the classic min-max, multiple Traveling Salesman Problem (TSP)\cite{TSPbook} which is NP-Hard. In min-max TSP, the mission time of each robot is incurred only due to travel; the objective is to find a sequence of targets to visit for each robot such that each target is visited at least once, the robots return to their initial location (or the depot), and the maximum travel time of any of the robots is minimized.  On the other hand, if there is no travel involved for the robots (say, all the targets are at the depot itself) and the only objective is to schedule the tasks with the human operators such that no operator works on more than one task at any time instant and the maximum completion time of any task must be minimized, the TASSP reduces to a parallel machine scheduling problem \cite{Graham1969} which is also NP-Hard. In TASSP, clearly, the sequence of targets visited by each robot influence how the corresponding tasks are scheduled with the human operators and vice-versa. Therefore, the TASSP is a challenging problem to solve.

The TASSP belongs to a rich class of problems that couple vehicle routing and machine scheduling. Here, we will discuss the main papers that are closely related to the TASSP. In \cite{how2010}, the authors consider a task allocation problem for a heterogeneous team of human operators and robots. They plan schedules based on a distributed approach while taking into account the workload on the human operators, the availability of robots and some coordination constraints. In \cite{Gombolay2018}, the authors address a joint task allocation and scheduling problem with coupled temporal and spatial constraints in the context of manufacturing applications. They present fast task sequencing heuristics and mixed integer linear programming based methods to find near-optimal solutions for real-world problems. In \cite{peters2018}, a joint robot routing/scheduling problem is considered for reducing the human operator's workload and the loitering times of the robots. The problem is first formulated as a mixed integer nonlinear program and then transformed to a mixed integer linear program, and based on these programs, a dynamic solution strategy is proposed to find sub-optimal schedules and routes.

A very generic version of the TASSP with heterogeneous vehicles, tasks with time window constraints and human operator specific scheduling constraints is presented in \cite{Murray13}. The authors in \cite{Murray13} also present a mathematical programming model and demonstrate the effectiveness of including human operator constraints in planning the routes of the vehicles. In \cite{Cai16}, a collaborative human-robot TSP is formulated where the objective is to maximize the target classification probability for a single robot while deciding on the targets that needs to be visited by the robot, the targets that need human assistance and the targets that can be classified by remote sensing. 

In this article, we are interested in developing an approximation algorithm for the TASSP. An $\alpha-$approximation algorithm for a problem is an algorithm that runs in polynomial time in the size of the input and produces a feasible solution for the problem that is at most $\alpha$ times the optimal cost for any instance of the problem. The quantity $\alpha$ is also referred to as the approximation ratio or approximation factor of the algorithm. For the TASSP, we are not aware of any approximation algorithms in the literature. 

Approximation algorithms help in two critical areas in solving problems such as the TASSP. First, as discussed earlier, the TASSP generalizes the min-max TSP which is known to be notoriously hard to solve for optimal solutions. The reason for this hardness comes from the fact that the lower bounds from standard formulations of the the min-max TSP are weak which in turn places a computational burden on any Branch and Bound or Branch and Cut procedure used to find optimal solutions\cite{Applegate2002}. For example, this is not the case for the multiple vehicle TSP where the objective is to minimize the sum of the travel costs of the vehicles\cite{sundar2016generalized,Sundar2017}. Even in the heterogeneous case, one can develop good bounds, and readily develop Branch and Cut solvers as we have done in \cite{Sundar2017}. Therefore, approximation algorithms such as the ones presented in this paper are very useful for solving the TASSP and provide bounded solutions with small computational cost. Second, the solution provided by an approximation algorithm can be used as an initial solution for meta-heuristics and further be improved upon \cite{levy2014heuristics}, or can be used to warm-start the Branch and Cut solvers as done in \cite{sundar2014algorithms}.

For the min-max TSP when all the robots start from a single depot, there is a well known approximation algorithm by Frederickson et al. \cite{Frederickson76}. This algorithm first finds a sequence of targets by first solving a single TSP using an approximation algorithm and then splits the sequence into $k$ sub-sequences (one for each robot) of more or less equal travel times. If $\beta$ is the approximation factor of the algorithm for the single TSP, then the approximation factor of Frederickson et al.'s algorithm \cite{Frederickson76} is $\alpha:=\beta+1-\frac{1}{k}$. If the travel times are symmetric and satisfy the triangle inequality, and Christofides heuristic \cite{Chris76} is used as the approximation algorithm for the single TSP, then $\beta:=\frac{3}{2}$, and therefore, $\alpha=\frac{5}{2} -\frac{1}{k}$. For the parallel machine scheduling problem, there is a classic greedy heuristic by Graham \cite{Graham1969} with an approximation ratio of $2-\frac{1}{m}$ where $m$ denotes the number of human operators (or machines). 

The main contribution of this article is an algorithm for the TASSP with approximation ratios equal to $\frac{7}{2} -\frac{5}{2k}$, $\frac{5}{2} -\frac{1}{k}$ and $\frac{7}{2} -\frac{1}{k}$ when $m=1$, $m\geq k\geq 2$ and $k>m\geq 2$ respectively\footnote{We do not consider the trivial case when $m\geq k=1$ where just one human operator can handle all the tasks corresponding to the robot. In this case, the TASSP simply reduces to the single TSP.}. We also present computational results to corroborate the performance of the proposed approximation algorithm. 

\section{Problem Statement}
Let $T:=\{t_1,t_2,\cdots,t_n\}$ denote the set of targets and $R:=\{r_1,r_2,\cdots,r_k\}$ represent the robots. Assume there are at least two robots ($k\geq 2$), and all of them are initially located at the depot $d$. Let $V:=T\cup\{d\}$. Let $c({i,j})$ denote the time required to travel between any two vertices $i,j, \in V$. The travel times are symmetric and satisfy the triangle inequality. That is, $c({i,j})=c({j,i})$ and $c({i,j})\leq c({i,k}) + c({k,j})$ for all $i,j,k\in V$.

Each target has a specific task that needs to be performed collaboratively between a robot and a human operator. We assume each task must be completed without preemption, $i.e.$, once a human operator starts working with a robot on a task, the task cannot be interrupted before it it processed completely. Suppose $p_i$ denotes the processing time needed to collaboratively complete the task at target $i\in T$. There are $m$ human operators available to perform the tasks and it is assumed that each human operator can handle at most one task at any time instant. 

A robot on arrival at a target may have to wait at the target if there is no human operator available to work with the robot, $i.e.$, all the human operators are busy working with other robots. In this case, the robot waits for the next available human operator to process the task. Therefore, the total mission time incurred by a robot is the sum of its travel time, wait time and the processing times of all the tasks at the targets visited by it. A schedule for a robot is a partitioning of its mission time into time intervals where each time interval corresponds to its travel time or wait time or task processing time. 

The objective of the problem is to allocate and sequence the targets (or tasks), and assign a schedule for each robot such that
\begin{itemize}
    \item each target is visited once by some robot, 
    \item at most $m$ tasks are processed by the human operators at any time instant, 
    \item each robot returns to the depot at the end of its mission, and,
    \item the maximum mission time of any of the robots is minimized.
\end{itemize}



\section{Algorithms for the TASSP} \label{sec:algorithms}

Prior to presenting the approximation algorithm, we present a greedy heuristic that can convert any solution, $F$, for the min-max TSP into a feasible solution for the TASSP. Let $F_r$ be the sequence of targets visited by  robot $r$ in solution $F$. Given $F:=\{F_{r_1},F_{r_2},\cdots, F_{r_k}\}$, the greedy heuristic works as follows: Starting at the depot, each robot will travel according to the sequence of targets assigned to it in $F$. When robot $r$ arrives at a target $u\in F_r$, any human operator, if available, is immediately assigned to work on the task at target $u$ with the robot $r$. If no human operator is available, the robot $r$ simply waits for the next operator to be available. Ties are broken arbitrarily when a human operator becomes available and multiple robots are waiting to collaborate. After all the targets are visited in $F_r$ and corresponding tasks completed, robot $r$ returns to the depot. 

We now present an approximation algorithm, $Approx$, for the TASSP. $Approx$ constructs two feasible solutions and chooses the solution with the least mission time for any robot. The steps in $Approx$ are as follows:
\begin{itemize}
    \item {\bf Algorithm 1}: Ignore the processing times at the targets and the scheduling constraints, and use Frederickson et al.'s algorithm \cite{Frederickson76} to solve the min-max TSP. Use the greedy heuristic presented above to convert the min-max TSP solution into a feasible solution for the TASSP. Let this solution be referred to as $S_1$. Let $Cost(S_1)$ denote the largest mission time of any robot in solution $S_1$.
    
    \item {\bf Algorithm 2}: For any edge $(i,j)$ joining targets $i,j\in T$, let the modified travel time be $\hat{c}({i,j}) = c({i,j}) + \frac{p_i}{2} + \frac{p_j}{2}$. For any edge $(i,d)$ joining a target $i$ and depot $d$, let the modified travel time be $\hat{c}({i,d}) = c({i,d}) + \frac{p_i}{2}$. Note that the modified travel times also are symmetric and satisfy the triangle inequality. Now, solve the min-max TSP with the modified travel times using the Frederickson et al.'s algorithm \cite{Frederickson76}.  Again, use the greedy heuristic to convert this min-max TSP solution into a feasible solution for the TASSP. Let this solution be referred to as $S_2$. Let $Cost(S_2)$ denote the largest mission time of any robot in solution $S_2$.
    
    \item Output $S\in \{S_1,S_2\}$ such that $Cost(S):=$ $\min\{Cost(S_1),$ $ Cost(S_2)\}$.
    
\end{itemize}

The following theorem is the main result of this article.

\begin{theorem}\label{theorem1}
The approximation ratio of Algorithm $Approx$ is equal to
\[ \left\{ \begin{array}{l}
    \frac{7}{2} -\frac{5}{2k}  \quad \text{if m=1},\\ 
     \\
    \frac{5}{2} -\frac{1}{k}  \quad ~\text{if m$\geq$ k$\geq 2$},\\
    \\
\frac{7}{2} -\frac{1}{k}  \quad ~\text{otherwise}.
  \end{array} \right.\]
\end{theorem}

\section{Proof of Theorem \ref{theorem1}}
The computational complexity of $Approx$ depends mainly on the approximation algorithm by Frederickson et al. \cite{Frederickson76} which can be implemented in the order of $n^3$ steps. Therefore, $Approx$ runs in polynomial time. In this section, we show the bounds which will provide the approximation ratio. We first prove a key result that holds true for any sequence of targets $F_r$ visited by robot $r$. This result (Lemma \ref{lemma1}) will bound the mission time for any of the robots in terms of its travel cost and the processing times of all the tasks. This result will be used to prove bounds for both the Algorithms 1 and 2 in $Approx$.

Let $F_r :=(u_{r1},u_{r2},\cdots,u_{rn_r})$ denote the sequence of targets visited by the robot $r$. In this sequence, $u_{r1}$ is the first target visited by robot $r$, $u_{r2}$ is the second target visited by robot $r$ and so on. Let $Cost^{travel}(F_r)$ denote the travel time for robot $r$ to visit the targets in $F_r$, $i.e.$, $Cost^{travel}(F_r):= c({d,u_{r1}}) + \sum_{i=1}^{n_r-1} c({u_{r_i},u_{r_{i+1}}}) + c({u_{n_r},d})$.

\begin{lemma}\label{lemma1}
 Given $F=\{F_{r_1},F_{r_2},\cdots,F_{r_k}\}$, let $Cost_r(F)$ denote the mission time of robot $r$ after applying the greedy heuristic on $F$. Then, 
\begin{equation}
    Cost_r(F) \leq Cost^{travel}(F_r)  + (1-\frac{1}{m})\sum_{i\in F_r} p_i + \frac{\sum_{i\in T} p_i}{m}.
\end{equation}
\end{lemma}

\begin{proof}
The total wait time for robot $r$ after applying the greedy heuristic is given by $Cost^{wait}_r(F):= Cost_r(F)-Cost^{travel}(F_r)-\sum_{i\in F_r} p_i $. During this wait time, all the human operators were busy working on tasks not assigned to robot $r$. Therefore, as there are $m$ human operators, $mCost^{wait}_r(F)$ must be at most equal to $\sum_{i\notin F_r} p_{i}$. Hence,
\begin{align}
    mCost^{wait}_r(F)  \leq \sum_{i\notin F_r} p_{i} & \nonumber \\
  m(Cost_r(F)-Cost^{travel}(F_r)-\sum_{i\in F_r} p_i)  \leq \sum_{i\notin F_r} p_{i} & \nonumber \\
  \Rightarrow Cost_r(F)  \leq Cost^{travel}(F_r)  + \sum_{i\in F_r} p_i  + & \frac{\sum_{i\notin F_r} p_{i}}{m} \nonumber \\
   Cost_r(F)  \leq Cost^{travel}(F_r) + (1-\frac{1}{m})\sum_{i\in F_r} p_i  + & \frac{\sum_{i\in T} p_{i}}{m}. \nonumber 
\end{align}
\end{proof}

In the next lemma, we related the cost of an optimal TSP solution and the processing times of the tasks with the optimal TASSP cost.

\begin{lemma}\label{lemma2}
Let $Cost^*_{tassp}(m)$ denote the optimal cost of the TASSP with $m$ human operators and let $Cost^*_{tsp}$ denote the cost of an optimal TSP tour visiting all the vertices in $V$. Then, $\frac{Cost^*_{tsp}}{k} + \frac{\sum_{i\in T}  p_{i}}{k} \leq Cost^*_{tassp}$.
\end{lemma}
\begin{proof}
Let $F^*_r$ be the sequence of targets visited by robot $r$ in an optimal solution for the TASSP. Let the corresponding travel cost for visiting the targets in $F^*_r$ be denoted by $Cost^{travel*}_r(F)$. We have, $Cost^{travel*}_r(F) + \sum_{i\in F^*_r} p_i \leq Cost^*_{tassp}(m)$. Summing over all the $k$ robots, we get the following inequality:
\[\sum_{r\in R}Cost^{travel*}_r + \sum_{i\in T} p_i \leq k Cost^*_{tassp}(m).\]

Now, consider a feasible solution for the single TSP where a robot visits all the targets in the sequence given by $(F^*_{r_1},F^*_{r_2},\cdots,F^*_{r_k})$. That is, the robot starts from the depot and visits the targets in the sequence $F^*_{r_1}$, then the targets in the sequence $F^*_{r_2} $ and so on, and finally returns to the depot. Clearly, the cost of this solution for the single TSP must be upper bounded by $\sum_{r\in R}Cost^{travel*}_r$ as the travel times satisfy the triangle inequality. Therefore, 
\begin{align}
    \frac{Cost^*_{tsp}}{k} + \frac{\sum_{i\in T}  p_{i}}{k} & \leq \frac{\sum_{r\in R}Cost^{travel*}_r}{k} + \frac{\sum_{i\in T}  p_{i}}{k} \nonumber \\ & \leq Cost^*_{tassp}(m). \nonumber
\end{align}
\end{proof}

The next three Lemmas prove the approximation ratios of the proposed algorithm. Prior to that, we present a bounding result from Frederickson et al.'s algorithm \cite{Frederickson76}
which will be used in the proof of the next Lemma. Let $C^{*}_{tsp}$ and $C^{*}_{mmtsp}$ denote the optimal cost corresponding to the single TSP and the min-max TSP respectively. Then, the travel cost of the feasible solution produced by the  Frederickson et al.'s algorithm (Theorem 4. in \cite{Frederickson76}) is upper bounded by

\begin{align}
\frac{3}{2k} Cost^*_{tsp}(m) + (1-\frac{1}{k})C^*_{mmtsp}. \label{BoundFrederickson}
\end{align}

\begin{lemma}\label{lemma:Algo1}
Let $F_r$ be the sequence of targets visited by the robot $r$ in Algorithm 1. Given $F=\{F_{r_1},F_{r_2},\cdots,F_{r_k}\}$, let $Cost_r(F)$ denote the mission time of robot $r$ for the feasible solution obtained using Algorithm 1. Then, $Cost(S_1) =\max_{r\in R} Cost_r(F) \leq (\frac{5}{2} + m -\frac{1}{k} - \frac{3m}{2k} )Cost^*_{tassp}(m)$. For $m=1$, the approximation ratio of Algorithm 1 is $\frac{7}{2} - \frac{5}{2k}$.
\end{lemma}

\begin{proof}
Let $r^*:= \arg \max_{r\in R} Cost_r(F)$. Using Lemmas \ref{lemma1},\ref{lemma2} and the upper bounding result in \eqref{BoundFrederickson}, we get,

\begin{align}
Cost(S_1) & =\max_{r\in R} Cost_r(F)\nonumber \\ & \leq Cost^{travel}(F_{r*})  + (1-\frac{1}{m})\sum_{i\in F_{r^*}} p_i + \frac{\sum_{i\in T} p_i}{m} \nonumber \\
& \leq \frac{3}{2k} C^*_{tsp} + (1-\frac{1}{k})C^*_{mmtsp}  \nonumber \\ &  + (1-\frac{1}{m})\sum_{i\in F_{r^*}} p_i  + \frac{\sum_{i\in T} p_i}{m} \nonumber \\
& \leq \frac{3}{2} (Cost^*_{tassp}(m)-\frac{\sum_{i\in T}  p_{i}}{k}) + (1-\frac{1}{k})C^*_{mmtsp}  \nonumber \\ &  + (1-\frac{1}{m})\sum_{i\in F_{r^*}} p_i + \frac{\sum_{i\in T} p_i}{m}. \nonumber 
\end{align}

Now, note that $\sum_{i\in F_{r^*}} p_i \leq \sum_{i\in T} p_i$. Therefore, 
\begin{align}
Cost(S_1) &  \leq \frac{3}{2} (Cost^*_{tassp}(m)-\frac{\sum_{i\in T}  p_{i}}{k}) + (1-\frac{1}{k})C^*_{mmtsp}  \nonumber \\ &  + (1-\frac{1}{m})\sum_{i\in T} p_i + \frac{\sum_{i\in T} p_i}{m} \nonumber \\
& = \frac{3}{2} Cost^*_{tassp}(m) + (1-\frac{1}{k})C^*_{mmtsp}  \nonumber \\ &  + (1-\frac{1}{m}-\frac{3}{2k})\sum_{i\in T} p_i + \frac{\sum_{i\in T} p_i}{m}. \label{bound1}
\end{align}

Note that $C^*_{mmtsp} \leq Cost^*_{tassp}(m)$ as the min-max TSP does include any processing times or scheduling constraints. Also, as there are $m$ human operators and the sum of the processing times of all the tasks is $\sum_{i\in T} p_i$, $\frac{\sum_{i\in T} p_i}{m}$ serves as a trivial lower bound for $Cost^*_{tassp}(m)$. Substituting these lower bounds in \eqref{bound1}, we get,

\begin{align}
Cost(S_1)  & \leq \frac{3}{2} Cost^*_{tassp}(m) + (1-\frac{1}{k}) Cost^*_{tassp}(m) \nonumber \\ &   + (1-\frac{1}{m}-\frac{3}{2k})m  Cost^*_{tassp}(m) +  Cost^*_{tassp}(m)\nonumber \\
& = (\frac{5}{2} + m -\frac{1}{k} - \frac{3m}{2k} )Cost^*_{tassp}(m). \nonumber
\end{align}

\end{proof}

\begin{lemma}\label{lemma:Algo21}
Assume the number of human operators is at least equal to the number of robots, $i.e.$, $m\geq k$. Let $F_r$ be the sequence of targets visited by the robot $r$ in Algorithm 2. Then, Algorithm 2 directly provides a feasible solution to the TASSP and has an approximation ratio of $\frac{5}{2} -\frac{1}{k}$. 
\end{lemma}
\begin{proof}
If $m \geq k$, it is possible to ensure there is no wait time for any robot independent of the sequences of targets assigned to the robots. This can be accomplished trivially by assigning each robot with a human operator while ensuring no operator is assigned to more than one robot. For this reason, there is always an optimal solution to the TASSP where the wait time of each robot is zero. Hence, if $m \geq k$, TASSP reduces to the problem where there are no scheduling constraints for the human operators and the objective is to find a sequence of targets for each robot such that each target is visited once by some robot and $\max_{r\in F_r }[Cost^{travel}(F_{r})  + \sum_{i\in F_{r}} p_i] $ is minimized. Note that the mission time for robot $r$ can also be rewritten in terms of the modified travel times as follows:
\begin{align}
    & Cost^{travel}(F_{r})  + \sum_{i\in F_{r}} p_i \nonumber \\ & = c({d,u_{r1}}) + \sum_{i=1}^{n_{r}-1} c({u_{ri},u_{r(i+1)}}) + c({u_{r n_{r}},d}) + \sum_{i\in F_{r}} p_i \nonumber \\
    & = c({d,u_{r1}}) + \frac{p_{u_{r1}}}{2} + \sum_{i=1}^{n_{r}-1} [c({u_{ri},u_{r(i+1)}}) + \frac{p_{u_{ri}}}{2} + \frac{p_{u_{r(i+1)}}}{2}] \nonumber \\ & + c({u_{rn_{r}},d}) + \frac{p_{u_{r n_r}}}{2} \nonumber \\
    & = \hat{c}({d,u_{r1}}) + \sum_{i=1}^{n_{r}-1} \hat{c}({u_{ri},u_{r{i+1}}}) + \hat{c}({u_{r n_{r}},d}).
\end{align}
Therefore, when $m\geq k$, TASSP reduces to solving a min-max TSP using the modified times to travel between any two vertices. As a result, applying Frederickson et al.'s algorithm \cite{Frederickson76} directly provides an approximation ratio of $\frac{5}{2} -\frac{1}{k}$. 
\end{proof}

\begin{lemma}\label{lemma:Algo22}
Assume the number of human operators is less than the number of robots, $i.e.$, $m < k$. Let $F_r$ be the sequence of targets visited by the robot $r$ in Algorithm 2. Given $F=\{F_{r_1},F_{r_2},\cdots,F_{r_k}\}$, let $Cost_r(F)$ denote the mission time of robot $r$ for the feasible solution obtained using Algorithm 2. Then, $Cost(S_2) =\max_{r\in R} Cost_r(F) \leq (\frac{7}{2} -\frac{1}{k} )Cost^*_{tassp}(m)$.
\end{lemma}
\begin{proof}
Let $r^*:= \arg \max_{r\in R} Cost_r(F)$. Using Lemma \ref{lemma1}, we get,

\begin{align}\label{lemma_inter}
Cost(S_2) & =\max_{r\in R} Cost_r(F)\nonumber \\
& \leq Cost^{travel}(F_{r*})  + (1-\frac{1}{m})\sum_{i\in F_{r^*}} p_i + \frac{\sum_{i\in T} p_i}{m} \nonumber \\
& \leq Cost^{travel}(F_{r*})  + \sum_{i\in F_{r^*}} p_i + \frac{\sum_{i\in T} p_i}{m}.
\end{align}
As discussed in Lemma \ref{lemma:Algo1}, 
\begin{align}\label{lemma_inter_1}
\frac{\sum_{i\in T} p_i}{m} \leq Cost^*_{tassp}(m).    
\end{align}. Also, using Lemma \ref{lemma:Algo21}, we get,

\begin{align}
   & Cost^{travel}(F_{r*}) + \sum_{i\in F_{r^*}} p_i \nonumber \\ & \leq \max_{r\in R} [Cost^{travel}(F_{r}) + \sum_{i\in F_{r}} p_i] \nonumber \\ & \leq  (\frac{5}{2} -\frac{1}{k})Cost^*_{tassp}(m){\rm~ for~ any~ }m\geq k.
\end{align}
For a fixed $k$, note that for any two positive integers $m_1,m_2$ such that $m_1\geq m_2$, $Cost^*_{tassp}(m_1)$ must be upper bounded by $Cost^*_{tassp}(m_2)$. This is due to the fact that any feasible solution to a TASSP with $m_2$ human operators can be trivially transformed into a feasible solution to a TASSP with $m_1$ human operators. Therefore,

\begin{align}\label{lemma_inter_2}
   & Cost^{travel}(F_{r*}) + \sum_{i\in F_{r^*}} p_i \nonumber \\ & \leq  (\frac{5}{2} -\frac{1}{k})Cost^*_{tassp}(m){\rm~ for~ any~ }m\geq k \nonumber \\ & \leq  (\frac{5}{2} -\frac{1}{k})Cost^*_{tassp}(m){\rm~ for~ any~ }m< k.
\end{align}

Substituting the bounds from equations \eqref{lemma_inter_1},\eqref{lemma_inter_2} in equation \eqref{lemma_inter} proves the Lemma. 

\end{proof}

The approximation ratios in the Theorem \ref{theorem1} directly follow from Lemmas \ref{lemma:Algo1}, \ref{lemma:Algo21} and \ref{lemma:Algo22}.

\begin{table*}[h]
    \centering
    \caption{{A-posteriori guarantee for small instances}}
    \resizebox{\textwidth}{!}{\begin{tabular}{c|c|c|c|c|c|c|c|c|c|c}
        \toprule
        \multirow{2}{*}{Instance No.} & \multirow{2}{*}{$|V|$} & \multirow{2}{*}{$k$} & \multirow{2}{*}{$m$}& \multicolumn{2}{|c|}{$Approx$} & \multicolumn{2}{|c|}{{Optimal}} & \multirow{2}{*}{Lower Bound} & \multicolumn{2}{c}{{$A-posteriori$ guarantee}}\\
        \cmidrule{5-8} \cmidrule{10-11}
         & & & & Cost & Run Time (secs) & Cost & Run Time (secs) & & Using optimal cost & Using lower bound \\
        \midrule

1  & 12  & 4  & 3  & 336.6  & 0.6    & 292.3  & 42.1  & 253.9 & 1.15 & 1.33 \\ 
2  & 12  & 4  & 3  & 216.8  & 0.5    & 216.5  & 39.1  & 175.9 & 1.00 & 1.23 \\ 
3  & 12  & 4  & 3  & 285.3  & 0.5    & 242.0  & 15.8  & 201.2 & 1.18 & 1.42\\ 
4  & 12  & 4  & 3  & 308.4  & 0.5    & 280.3  & 23.3  & 224.8 & 1.10 & 1.37\\ 
5  & 12  & 4  & 3  & 259.5  & 0.5    & 199.4  & 28.8  & 174.7 & 1.30 & 1.49\\ 
6  & 12  & 4  & 3  & 304.1  & 0.5    & 232.0  & 22.3  & 168.6 & 1.31 & 1.8\\ 
7  & 12  & 4  & 3  & 313.9  & 0.5    & 304.5  & 179.0  & 268.2 & 1.03 & 1.17\\ 
8  & 12  & 4  & 3  & 219.9  & 0.5    & 185.4  & 20.5  & 162.0  & 1.19 & 1.36\\ 
9  & 12  & 4  & 3  & 243.3  & 0.5    & 210.1  & 12.4  & 156.3  & 1.16 & 1.56\\ 
10  & 12  & 4  & 3  & 320.9  & 0.5   & 276.8  & 151.4  & 251.2 & 1.16 & 1.28\\ 
        \midrule
   
11  & 12  & 3  & 2  & 290.3  & 0.5    & 279.5  & 352.6  & 211.5 & 1.04 & 1.37 \\ 
12  & 12  & 3  & 2  & 295.0  & 0.5    & 255.4  & 221.0  & 190.1 & 1.15 & 1.55 \\ 
13  & 12  & 3  & 2  & 318.9  & 0.5    & 278.5  & 580.3  & 215.5 & 1.14 & 1.48 \\ 
14  & 12  & 3  & 2  & 340.6  & 0.5    & 319.4  & 216.2  & 230.3 & 1.07 & 1.48 \\ 
15  & 12  & 3  & 2  & 261.4  & 0.5    & 225.3  & 151.6  & 181.4 & 1.16 & 1.44 \\ 
16  & 12  & 3  & 2  & 339.2  & 0.5    & 319.6  & 139.6  & 252.2 & 1.06 & 1.35 \\ 
17  & 12  & 3  & 2  & 324.1  & 0.5    & 284.7  & 236.9  & 212.8 & 1.14 & 1.52 \\ 
18  & 12  & 3  & 2  & 230.6  & 0.5    & 225.4  & 262.9  & 165.5 & 1.02 & 1.39 \\ 
19  & 12  & 3  & 2  & 374.5  & 0.5    & 309.9  & 105.3  & 238.4 & 1.21 & 1.57 \\ 
20  & 12  & 3  & 2  & 352.5  & 0.5    & 335.0  & 338.3  & 266.6 & 1.05 & 1.32 \\ 
        \midrule

21  & 10  & 3  & 2  & 242.1  & 0.51    & 230.3  & 8.4  & 153.0  & 1.05   & 1.58\\ 
22  & 10  & 3  & 2  & 279.6  & 0.48    & 249.3  & 7.8  & 187.1  & 1.12   & 1.49\\ 
23  & 10  & 3  & 2  & 286.9  & 0.48    & 271.4  & 5.6  & 197.5  & 1.06   & 1.45\\ 
24  & 10  & 3  & 2  & 279.4  & 0.48    & 246.0  & 6.9  & 189.6  & 1.14   & 1.47\\ 
25  & 10  & 3  & 2  & 277.4  & 0.48    & 269.3  & 6.5  & 210.7  & 1.03   & 1.32\\ 
26  & 10  & 3  & 2  & 231.8  & 0.48    & 210.2  & 6.0  & 170.3  & 1.10    & 1.36\\ 
27  & 10  & 3  & 2  & 259.7  & 0.49    & 259.7  & 6.7  & 231.7  & 1.00    & 1.12\\ 
28  & 10  & 3  & 2  & 235.8  & 0.48    & 225.8  & 5.4  & 175.4  & 1.04   & 1.34\\ 
29  & 10  & 3  & 2  & 268.2  & 0.48    & 261.1  & 3.5  & 232.9  & 1.03   & 1.15\\ 
30  & 10  & 3  & 2  & 286.7  & 0.48    & 250.3  & 6.3  & 203.8  & 1.15   & 1.41\\ 
        \bottomrule
    \end{tabular}}
    \label{tab:solution}
\end{table*}

\begin{table*}
    \centering
    \caption{A-posteriori guarantee for larger instances}
    \resizebox{\textwidth}{!}{\begin{tabular}{c|c|c|c|c|c|c|c|c|c|c|c|c}
        \toprule
        \multirow{2}{*}{Instance No.} & \multirow{2}{*}{$|V|$} & \multirow{2}{*}{$k$} & \multicolumn{5}{|c|}{Cost of the solution using $Approx$} &  \multicolumn{5}{c}{{$A-posteriori$ guarantee using lower bounds}}\\
        \cmidrule{4-8} \cmidrule{9-13}
         & & & $m=1$ & $m=2$ & $m=3$ & $m=4$ & $m=5$  & $m=1$ & $m=2$ & $m=3$ & $m=4$ & $m=5$ \\
        \midrule
31 & 50 & 5 & 692.3 & 380.6 & 283.2 & 268.4 & 259.4 & 1.04 & 1.14 & 1.27 & 1.39 & 1.35  \\ 
32 & 50 & 5 & 691.0 & 374.3 & 286.0 & 255.1 & 249.6 & 1.04 & 1.13 & 1.29 & 1.33 & 1.30  \\ 
33 & 50 & 5 & 680.1 & 384.0 & 282.1 & 270.5 & 266.3 & 1.03 & 1.16 & 1.28 & 1.40 & 1.38  \\ 
34 & 50 & 5 & 640.5 & 358.4 & 284.6 & 254.8 & 251.8 & 1.04 & 1.17 & 1.39 & 1.42 & 1.40  \\ 
35 & 50 & 5 & 637.9 & 338.0 & 255.0 & 227.1 & 226.4 & 1.04 & 1.10 & 1.24 & 1.26 & 1.26  \\ 
36 & 50 & 5 & 626.6 & 345.8 & 275.4 & 257.7 & 246.4 & 1.04 & 1.14 & 1.37 & 1.47 & 1.40  \\ 
37 & 50 & 5 & 651.8 & 352.0 & 264.6 & 245.9 & 240.9 & 1.03 & 1.11 & 1.26 & 1.38 & 1.35  \\ 
38 & 50 & 5 & 655.2 & 352.4 & 271.4 & 242.0 & 233.9 & 1.03 & 1.10 & 1.28 & 1.36 & 1.31  \\ 
39 & 50 & 5 & 688.7 & 381.1 & 288.4 & 264.4 & 256.8 & 1.04 & 1.15 & 1.30 & 1.38 & 1.34  \\ 
40 & 50 & 5 & 695.3 & 384.7 & 292.7 & 269.7 & 260.2 & 1.05 & 1.16 & 1.32 & 1.40 & 1.35  \\
41 & 50 & 5 & 632.5 & 349.7 & 274.6 & 247.1 & 238.5 & 1.05 & 1.16 & 1.37 & 1.37 & 1.32  \\ 
42 & 50 & 5 & 670.5 & 360.1 & 271.1 & 261.9 & 253.2 & 1.01 & 1.09 & 1.23 & 1.38 & 1.33  \\ 
43 & 50 & 5 & 647.4 & 360.1 & 291.0 & 255.4 & 250.5 & 1.02 & 1.13 & 1.37 & 1.39 & 1.37  \\ 
44 & 50 & 5 & 606.5 & 341.7 & 253.8 & 238.5 & 238.5 & 1.02 & 1.15 & 1.28 & 1.34 & 1.34  \\ 
45 & 50 & 5 & 745,9 & 425.8 & 308.7 & 282.9 & 280.5 & 1.01 & 1.16 & 1.26 & 1.39 & 1.38  \\ 
46 & 50 & 5 & 682.5 & 373.2 & 282.4 & 252.7 & 251.1 & 1.03 & 1.13 & 1.28 & 1.33 & 1.32  \\ 
47 & 50 & 5 & 709.6 & 404.8 & 317.6 & 294.3 & 288.8 & 1.07 & 1.22 & 1.44 & 1.54 & 1.51  \\ 
48 & 50 & 5 & 746.0 & 414.2 & 316.1 & 306.7 & 286.5 & 1.05 & 1.16 & 1.33 & 1.51 & 1.42  \\ 
49 & 50 & 5 & 631.9 & 351.3 & 269.7 & 247.2 & 244.2 & 1.03 & 1.15 & 1.33 & 1.38 & 1.36  \\ 
50 & 50 & 5 & 679.3 & 371.9 & 292.6 & 262.5 & 261.5 & 1.04 & 1.14 & 1.34 & 1.39 & 1.38  \\ 
51 & 50 & 5 & 673.7 & 374.2 & 288.8 & 272.9 & 263.8 & 1.04 & 1.16 & 1.34 & 1.46 & 1.41  \\ 
52 & 50 & 5 & 630.6 & 333.8 & 262.3 & 235.4 & 229.5 & 1.05 & 1.11 & 1.30 & 1.34 & 1.30  \\ 
53 & 50 & 5 & 636.2 & 339.1 & 262.0 & 241.0 & 236.9 & 1.04 & 1.11 & 1.28 & 1.38 & 1.36  \\ 
54 & 50 & 5 & 671.7 & 371.1 & 277.2 & 255.0 & 251.0 & 1.09 & 1.21 & 1.35 & 1.41 & 1.38  \\ 
55 & 50 & 5 & 639.2 & 350.2 & 268.7 & 243.0 & 241.6 & 1.04 & 1.14 & 1.31 & 1.34 & 1.33  \\ 

        \bottomrule
    \end{tabular}}
    \label{tab:49targets}
\end{table*}

\section{Simulations}
In this section, we implement the approximation algorithm on a set of instances to infer the a-posteriori guarantees, $i.e.$, for a given instance, the a-posteriori guarantee is defined as the ratio of the cost of the feasible solution obtained by the approximation algorithm and the optimal cost. These guarantees are generally lower than the approximation ratio which is a (worst case) theoretical bound true for any instance of the problem. The optimal cost was obtained by using a Mixed Integer Linear Programming Formulation of the TASSP given in the appendix. This formulation was implemented in Gurobi \cite{gurobi} and worked for small instances\footnote{For some instances with 15 targets, the run time in Gurobi was in the order of days.} (up to 12 vertices, 4 vehicles and 3 operators). For larger instances, we used the maximum of the lower bounds given by $L_1:=\frac{Cost^*_{tsp} + \sum_{i\in T}  p_{i} }{k} $ (Lemma \ref{lemma2}), the trivial bounds $L_2:=\frac{\sum_{i\in T}p_i}{m}$ and $L_3:= \max_{i\in T}(2c(d,i) + p_i)$ as a proxy for the optimal cost. 

We first generated 30 instances with at most $|V|=12$ vertices (targets and the depot), $k=4$ robots and $m=3$ human operators. The location of the vertices were sampled from an area of size $50 \times 50$ units. The speed of each robot was assumed to be 1 unit, and therefore, the travel time between any two vertices was set to be equal to the Euclidean distance between the vertices. Given an instance, the processing time for each target was sampled from a normal distribution with a mean value of $50 \%$ of the average travel times and a standard deviation of $2 \%$ of the average travel times. $Approx$ was coded in Julia \cite{bezanson2017julia} with the help of the NetworkX package, \cite{hagberg2008exploring} and the computations were run on Mac Pro (8-Core Intel Xeon E5 processor @3 GHz, 32 GB RAM). The cost of the solution obtained by $Approx$, the optimal cost and their corresponding computation times are  shown in Table \ref{tab:solution}. The approximation ratios as calculated by Theorem \ref{theorem1} for these instances are 3.25 ($=\frac{7}{2}-\frac{1}{4}$) and 3.16 ($=\frac{7}{2}-\frac{1}{3}$) for instances 1-10 and 11-30 respectively. As expected the a-posteriori guarantees as shown in Table \ref{tab:solution} were much lower than the approximation ratios. The average of the a-posteriori guarantees computed using the optimal costs was equal to $\approx$ 1.1. On the other hand, the average of the a-posteriori guarantees computed using the lower bounds was equal to $\approx$ 1.4. The average run time of $Approx$ was $\approx 0.6$ secs for the instances in this set while the average run time of the formulation in Gurobi varied significantly with an average of $\approx 106.75$ secs. 

The computational results for the larger instances are shown in Table \ref{tab:49targets}. In these instances, $|V|=50$, $k=5$ and the number of operators ($m$) was increased from 1 to $k$. The locations of the vertices and the processing times of the corresponding tasks were generated using the same procedure as before. As we could not compute optimal solutions for these instances, the a-posteriori guarantees were computed using the lower bounds. These results are shown in Table  \ref{tab:49targets}. As expected, for a given instance, the cost of the feasible solution monotonically decreased with $m$  (as discussed in equation \eqref{lemma_inter_2}, Lemma \ref{lemma:Algo22}). Correspondingly, we also observed that the a-posteriori guarantees increased with $m$. One of the reasons this occurred is because the lower bounds got poorer; specifically, for all the instances corresponding to $m=1$, $m=2$ and $m=3$, the lower bound computed using $L_2$ was binding ($i.e.$, $L_2=\max\{L_1,L_2,L_3\}$) whereas for $m=4$ and $m=5$, the lower bounds computed using $L_1$ was binding. As $L_2$ is inversely proportional to $m$, it did not contribute much for higher values of $m$. 

\section{Conclusions}
This article considered a task allocation, sequencing and scheduling problem for a team of human operators and robots. An approximation algorithm which combines ideas from vehicle routing and scheduling theory was presented. Computational results was also presented to corroborate the performance of the approximation algorithm. Future work can focus on developing exact algorithms and better lower bounds. Algorithms addressing uncertainties with respect to task processing times or travel times for the robots will also be useful in practical applications.

\bibliographystyle{IEEEtran}
\bibliography{references.bib}

\begin{thebibliography}{10}
\providecommand{\url}[1]{#1}
\csname url@samestyle\endcsname
\providecommand{\newblock}{\relax}
\providecommand{\bibinfo}[2]{#2}
\providecommand{\BIBentrySTDinterwordspacing}{\spaceskip=0pt\relax}
\providecommand{\BIBentryALTinterwordstretchfactor}{4}
\providecommand{\BIBentryALTinterwordspacing}{\spaceskip=\fontdimen2\font plus
\BIBentryALTinterwordstretchfactor\fontdimen3\font minus
  \fontdimen4\font\relax}
\providecommand{\BIBforeignlanguage}[2]{{%
\expandafter\ifx\csname l@#1\endcsname\relax
\typeout{** WARNING: IEEEtran.bst: No hyphenation pattern has been}%
\typeout{** loaded for the language `#1'. Using the pattern for}%
\typeout{** the default language instead.}%
\else
\language=\csname l@#1\endcsname
\fi
#2}}
\providecommand{\BIBdecl}{\relax}
\BIBdecl

\bibitem{cummings2010}
\BIBentryALTinterwordspacing
M.~Cummings, A.~Clare, and C.~Hart, ``The role of human-automation consensus in
  multiple unmanned vehicle scheduling,'' \emph{Human Factors}, vol.~52, no.~1,
  pp. 17--27, 2010, pMID: 20653222. [Online]. Available:
  \url{https://doi.org/10.1177/0018720810368674}
\BIBentrySTDinterwordspacing

\bibitem{Murray13}
C.~C. {Murray} and W.~{Park}, ``Incorporating human factor considerations in
  unmanned aerial vehicle routing,'' \emph{IEEE Transactions on Systems, Man,
  and Cybernetics: Systems}, vol.~43, no.~4, pp. 860--874, July 2013.

\bibitem{cummings12}
\BIBentryALTinterwordspacing
L.~Bertuccelli, W.~Beckers, and M.~Cummings, ``Developing operator models for
  uav search scheduling,'' in \emph{AIAA Guidance, Navigation, and Control
  Conference}, 2012. [Online]. Available:
  \url{https://arc.aiaa.org/doi/abs/10.2514/6.2010-7863}
\BIBentrySTDinterwordspacing

\bibitem{golden2018}
\BIBentryALTinterwordspacing
A.~Otto, N.~Agatz, J.~Campbell, B.~Golden, and E.~Pesch, ``Optimization
  approaches for civil applications of unmanned aerial vehicles (uavs) or
  aerial drones: A survey,'' \emph{Networks}, vol.~72, no.~4, pp. 411--458,
  2018. [Online]. Available:
  \url{https://onlinelibrary.wiley.com/doi/abs/10.1002/net.21818}
\BIBentrySTDinterwordspacing

\bibitem{Cai16}
H.~{Cai} and Y.~{Mostofi}, ``A human-robot collaborative traveling salesman
  problem: Robotic site inspection with human assistance,'' in \emph{2016
  American Control Conference (ACC)}, July 2016, pp. 6170--6176.

\bibitem{agrisurvey}
\BIBentryALTinterwordspacing
T.~Duckett, S.~Pearson, S.~Blackmore, and B.~Grieve, ``Agricultural robotics:
  The future of robotic agriculture,'' \emph{CoRR}, vol. abs/1806.06762, 2018.
  [Online]. Available: \url{http://arxiv.org/abs/1806.06762}
\BIBentrySTDinterwordspacing

\bibitem{agri1}
C.~Peña~Cortés, C.~Riaño~Jaimes, and G.~Moreno, ``Robotgreen. a teleoperated
  agricultural robot for structured environments,'' \emph{journal of
  Engineering Science and Technology Review}, vol.~11, pp. 145--155, 12 2018.

\bibitem{TSPbook}
G.~Gutin and A.~P. Punnen, Eds., \emph{The traveling salesman problem and its
  variations}, ser. Combinatorial optimization.\hskip 1em plus 0.5em minus
  0.4em\relax Dordrecht, London: Kluwer Academic, 2002.

\bibitem{Graham1969}
\BIBentryALTinterwordspacing
R.~Graham, ``Bounds on multiprocessing timing anomalies,'' \emph{SIAM Journal
  on Applied Mathematics}, vol.~17, no.~2, pp. 416--429, 1969. [Online].
  Available: \url{https://doi.org/10.1137/0117039}
\BIBentrySTDinterwordspacing

\bibitem{how2010}
\BIBentryALTinterwordspacing
S.~Ponda, H.-L. Choi, and J.~How, \emph{Predictive Planning for Heterogeneous
  Human-Robot Teams}. [Online]. Available:
  \url{https://arc.aiaa.org/doi/abs/10.2514/6.2010-3349}
\BIBentrySTDinterwordspacing

\bibitem{Gombolay2018}
M.~C. {Gombolay}, R.~J. {Wilcox}, and J.~A. {Shah}, ``Fast scheduling of robot
  teams performing tasks with temporospatial constraints,'' \emph{IEEE
  Transactions on Robotics}, vol.~34, no.~1, pp. 220--239, Feb 2018.

\bibitem{peters2018}
\BIBentryALTinterwordspacing
J.~R. Peters, A.~Surana, and F.~Bullo, ``Robust scheduling and routing for
  collaborative human/unmanned aerial vehicle surveillance missions,''
  \emph{Journal of Aerospace Information Systems}, vol.~15, no.~10, pp.
  585--603, 2018. [Online]. Available: \url{https://doi.org/10.2514/1.I010560}
\BIBentrySTDinterwordspacing

\bibitem{Applegate2002}
\BIBentryALTinterwordspacing
D.~Applegate, W.~Cook, S.~Dash, and A.~Rohe, ``Solution of a min-max vehicle
  routing problem,'' \emph{INFORMS Journal on Computing}, vol.~14, no.~2, pp.
  132--143, 2002. [Online]. Available:
  \url{https://doi.org/10.1287/ijoc.14.2.132.118}
\BIBentrySTDinterwordspacing

\bibitem{sundar2016generalized}
K.~Sundar and S.~Rathinam, ``Generalized multiple depot traveling salesmen
  problem-polyhedral study and exact algorithm,'' \emph{Computers \& Operations
  Research}, vol.~70, pp. 39--55, 2016.

\bibitem{Sundar2017}
\BIBentryALTinterwordspacing
------, ``Algorithms for heterogeneous, multiple depot, multiple unmanned
  vehicle path planning problems,'' \emph{Journal of Intelligent {\&} Robotic
  Systems}, vol.~88, no.~2, pp. 513--526, Dec 2017. [Online]. Available:
  \url{https://doi.org/10.1007/s10846-016-0458-5}
\BIBentrySTDinterwordspacing

\bibitem{levy2014heuristics}
D.~Levy, K.~Sundar, and S.~Rathinam, ``Heuristics for routing heterogeneous
  unmanned vehicles with fuel constraints,'' \emph{Mathematical Problems in
  Engineering}, vol. 2014, 2014.

\bibitem{sundar2014algorithms}
K.~Sundar and S.~Rathinam, ``Algorithms for routing an unmanned aerial vehicle
  in the presence of refueling depots,'' \emph{IEEE Transactions on Automation
  Science and Engineering}, vol.~11, no.~1, pp. 287--294, 2014.

\bibitem{Frederickson76}
G.~N. {Frederickson}, M.~S. {Hecht}, and C.~E. {Kim}, ``Approximation
  algorithms for some routing problems,'' in \emph{17th Annual Symposium on
  Foundations of Computer Science (sfcs 1976)}, Oct 1976, pp. 216--227.

\bibitem{Chris76}
N.~Christofides, ``Worst-case analysis of a new heuristic for the travelling
  salesman problem,'' Graduate School of Industrial Administration, Carnegie
  Mellon University, Technical Report, 1976.

\bibitem{gurobi}
\BIBentryALTinterwordspacing
L.~Gurobi~Optimization, ``Gurobi optimizer reference manual,'' 2019. [Online].
  Available: \url{http://www.gurobi.com}
\BIBentrySTDinterwordspacing

\bibitem{bezanson2017julia}
J.~Bezanson, A.~Edelman, S.~Karpinski, and V.~B. Shah, ``Julia: A fresh
  approach to numerical computing,'' \emph{SIAM review}, vol.~59, no.~1, pp.
  65--98, 2017.

\bibitem{hagberg2008exploring}
A.~Hagberg, P.~Swart, and D.~S~Chult, ``Exploring network structure, dynamics,
  and function using networkx,'' Los Alamos National Lab.(LANL), Los Alamos, NM
  (United States), Tech. Rep., 2008.

\bibitem{mccormick1976computability}
G.~P. McCormick, ``Computability of global solutions to factorable nonconvex
  programs: Part i—convex underestimating problems,'' \emph{Mathematical
  programming}, vol.~10, no.~1, pp. 147--175, 1976.

\end{thebibliography}

 \section{Appendix - Mixed Integer Linear Program}
\noindent \textbf{Decision Variables:} For all $i,j\in V, i\neq j$, the binary variable $x_{i,j}$ is equal to 1 if a robot is traveling from $i$ to $j$ and is equal to 0 otherwise. Similarly, for all $u,v\in V$, $u\neq v$, $y_{u,v}$ is equal to 1 if an operator works on the task at $v$ immediately after the task at $u$, and is equal to zero otherwise. There are also continuous variables defined as follows:

\noindent $vet_i$ = time at which a robot reaches target $i \in T$.\\
$vlt_i$ = time at which a robot  leaves target $i \in T$.\\
$tst_i$ = time at which the processing of task at $i \in T$ is started.\\
$MT$ = maximum mission time of any robot.\\

\noindent \textbf{Objective:}
The objective is to minimize the maximum mission time of any robot, $i.e.$, $\min MT$.

\noindent \textbf{Constraints:}\\
Number of incoming and outgoing edges from the depot is equal to the number of robots. That is,
\begin{equation}
     \sum_{j \in T} x_{j,d} = \sum_{j \in T} x_{d,j} = k.
\end{equation}    

There is exactly one incoming and outgoing travel edge for each target. That is,
\begin{equation}
    \sum_{i \in V\setminus{j}} x_{i,j} = \sum_{i \in V\setminus{j}} x_{j,i} = 1~ \forall j \in T.
\end{equation}  

The depot is considered as a dummy vertex from which the operators leave for processing the tasks and return after completing them. That is,

\begin{equation}
    \sum_{j \in T} y_{d,j} = \sum_{j \in T} y_{j,d} = m.
\end{equation}    

An operator can work on at most one other task immediately after completing a task at a target. That is,
\begin{equation}
    \sum_{i \in V\setminus{j}} y_{i,j} = \sum_{i \in V\setminus{j}} y_{j,i} = 1~ \forall j \in T.
\end{equation}    

The earliest a robot can reach a target is at least equal to the sum of the completion time of its previous task (if any) and the time taken to travel from its previous vertex. That is,

\begin{equation}
    (vet_j - tst_i)x_{i,j} \geq (c(i,j) + p_i)x_{i,j}~~ \forall i \in T, j \in T\setminus{i}.
    \label{eq:mc1}
\end{equation}

The earliest an operator can process a task is after completing his/her prior task if any. That is,
\begin{equation}
    (tst_j - tst_i)y_{i,j} \geq p_iy_{i,j}~~ \forall i \in T, j \in T\setminus{i}.
    \label{eq:mc2}
\end{equation}

The earliest the robot can reach a target is at least equal to the time taken for it to travel from the depot to the target. That is,
\begin{equation}
    vet_i \geq c(d,i)x_{d,i}~  \forall i \in T.
\end{equation}

The earliest a task can be processed is at least after the robot has reached the target.
\begin{equation}
    tst_i \geq vet_i~ \forall i \in T.
\end{equation}


For any target, the maximum mission time, $MT$, must be at least equal to the sum of the task completion time at the target and the time taken to return from the target to the depot. That is,
\begin{equation}
    tst_i + p_i + c(i,d) x_{i,d} \; \leq MT ~~ \forall i \in T.
\end{equation}


Though constraints \eqref{eq:mc1}--\eqref{eq:mc2} contain bi-linear terms, they can be linearized using the McCormick relaxation \cite{mccormick1976computability} without loosing any information, as the relaxation is exact for bi-linear variables. So, the formulation can be recast as a Mixed Integer Linear Program (MILP). In the interest of space, we omit the relaxed constraints here.

\end{document}